%% file: main.tex
\documentclass{article}
\usepackage{spconf,amsmath,graphicx,hyperref}
\usepackage{amssymb}
\usepackage{amsthm}
\usepackage{algorithm}
\usepackage{algpseudocode}
\usepackage{xcolor}
\usepackage{bm}
\usepackage{balance}

\newtheorem{proposition}{Proposition}
\newtheorem{assumption}{Assumption}

\usepackage[capitalize]{cleveref} 
\crefname{theorem}{Thm.}{Thms.}
\crefname{proposition}{Prop.}{Props.}
\crefname{lemma}{lem.}{lems.}
\crefname{corollary}{Cor.}{Cors.}
\crefname{definition}{Def.}{Defs.}
\crefname{section}{Sec.}{Secs.}
\crefname{figure}{Fig.}{Figs.}
\crefname{problem}{Prob.}{Probs.}
\crefname{appendix}{App.}{Apps.}
\crefname{equation}{Eq.}{Eqs.}
\crefname{algorithm}{Alg.}{Algs.}
\crefname{assumption}{Ass.}{Ass.}

\floatstyle{ruled}
\newfloat{algorithm}{tbp}{loa}
\providecommand{\algorithmname}{Algorithm}
\floatname{algorithm}{\protect\algorithmname}
\usepackage{stfloats}  

\newcommand{\spara}[1]{\smallskip\noindent\textbf{#1}}

\title{Simplicial Gaussian Models: Representation and Inference}
\name{Lorenzo Marinucci$^{1}$, Gabriele D'Acunto$^{2,3}$, Paolo Di Lorenzo$^{2,3}$, Sergio Barbarossa$^{2}$ \vspace{-.2cm}
}
\address{
    $^{1}$Statistical Sciences Dept., Sapienza University, Rome, Italy \\
    $^{2}$Information Engineering, Electronics and Telecommunications Dept., Sapienza University, Rome, Italy \\
    $^{3}$National Inter-University Consortium for Telecommunications, Parma, Italy\\
    \{l.marinucci, gabriele.dacunto, paolo.dilorenzo, sergio.barbarossa\}@uniroma1.it
\vspace{-.4cm}}
\begin{document}
\ninept
\maketitle

\input{Sections/abstract}
\begin{keywords}
Topological signal processing, simplicial complexes, Gaussian Markov random fields, probabilistic modeling.
\end{keywords}
\input{Sections/introduction}
\input{Sections/background}
\input{Sections/simplicial_gaussian_model}

\input{Sections/numerical_results}
\input{Sections/conclusions}

\clearpage
\spara{Acknowledgments.} The work was supported by the SNS JU project 6G-GOALS under the EU’s Horizon program Grant Agreement No 101139232, by Huawei Technology France SASU under Grant N. Tg20250616041, and by the  European Union under the Italian National Recovery and Resilience Plan of NextGenerationEU, partnership on Telecommunications of the Future (PE00000001- program RESTART). 
\balance

\bibliographystyle{IEEEbib}
\bibliography{strings,refs}

\end{document}

%% file: Sections/abstract.tex
\begin{abstract}
Probabilistic graphical models (PGMs) are powerful tools for representing statistical dependencies through graphs in high-dimensional systems.
However, they are limited to pairwise interactions.  
In this work, we propose the \emph{simplicial Gaussian model} (SGM), which extends Gaussian PGM to simplicial complexes. 
SGM jointly models random variables supported on vertices, edges, and triangles, within a single parametrized Gaussian distribution.  
Our model builds upon discrete Hodge theory and incorporates uncertainty at every topological level through independent random components.  
Motivated by applications, we focus on the marginal edge-level distribution while treating node- and triangle-level variables as latent.
We then develop a maximum-likelihood inference algorithm to recover the parameters of the full SGM and the induced conditional dependence structure.
Numerical experiments on synthetic simplicial complexes with varying size and sparsity confirm the effectiveness of our algorithm.
\end{abstract}

%% file: Sections/introduction.tex
\vspace{-0.2cm}\section{Introduction}
\label{sec:intro}
Probabilistic Graphical Models (PGMs) allows to represent and reason on high-dimensional systems under uncertainty \cite{Koller}, offering compact representations of complex joint distributions that combine probability with graph theory and enable efficient inference and learning algorithms. 
Thus, they are widely used in several applications, including computer vision, computational biology, and spatial statistics \cite{bishop,scutari,lindgren2011spde}. 
In a PGM, random variables are associated with the vertices of a graph, while edges encode statistical dependencies. 
The meaning of the edges depend on the graph type: Bayesian Networks capture directional dependencies through directed acyclic graphs (DAGs) \cite{pearl88}, whereas Markov Random Fields (MRFs) model symmetric conditional dependencies with undirected graphs,  thanks to the \emph{Markov property} \cite{besag1974spatial}. 
A well-studied family is Gaussian Markov Random Fields (GMRFs), i.e., MRFs that model Gaussian random variables \cite{rue2005gaussian}. 
Indeed, conditional dependencies in the Gaussian distribution are encoded by the precision matrix, thus allowing to learn GMRF from data with efficient algorithms \cite{friedman2008glasso}.

However, PGMs are inherently limited to graphs. 
First, PGMs typically associate random variables with individual nodes (sets of cardinality one), while in many settings random quantities naturally relates with larger sets.
Examples include data traffic in communication networks or water flows in distribution networks, where measurements are collected on the links of the networks \cite{lambiotte2019higherorder,schaub2021signal,cattai2025physicsinformedtopologicalsignalprocessing}. 
Second, PGMs are restricted to modeling pairwise dependencies via edges. 
However, many complex systems exhibit interactions involving groups of more than two variables, such as simultaneous activations of neurons in brain networks or flow conservation laws in power grids \cite{Courtney_2016,kanari2018topological,petri2014homological}. 
Such higher-order interactions cannot be effectively captured by a graph, motivating the use of higher-order topological descriptors. 
Simplicial complexes are a class of hypergraphs that address these limitations and have found applications in signal processing, applied algebraic topology, and machine learning \cite{barbarossa2020topological,carlsson2009tda,battiloro2023gsan,marinucci2025}. 
They generalize graphs by including higher-dimensional building blocks called simplices, which allow encoding multi-way interactions. 
Further, their hierarchical structure consistently relates simplices of adjacent dimension through incidence relations. 
Notably, simplicial-based representations have solid algebraic foundation corresponding to the discrete Hodge theory \cite{Eckmann1944/45}, which enables a principled treatment of data supported on simplices of different orders. 
For these reasons, we propose \emph{simplical Gaussian models} (SGMs) to extend GMRFs to simplicial complexes. \\ 
\noindent\textbf{Related works.} The bulk of existing work on simplicial and cell complexes focuses on deterministic settings \cite{isufi2025topological}.
Conversely, SGMs models uncertainty and stochastic variability. 
To date, only a limited number of contributions address probabilistic modeling on these domains \cite{sardellitti2022probabilistic, yang2024hodgeedgegp,alain2024gaussian,gurugubelli2024simplicial}. In detail, the work in \cite{sardellitti2022probabilistic} presents a probabilistic framework over simplicial complexes together with an inference algorithm, referred to as the simplicial lasso, to estimate the statistical dependencies between random variables supported on edges. 
Authors in \cite{yang2024hodgeedgegp,alain2024gaussian} introduce a class of Matérn fields defined over simplicial and cell complexes, extending the classical Matérn kernels for vector fields. 
In \cite{gurugubelli2024simplicial} the authors propose a probabilistic method to infer the structure of a simplicial complex from random edge observations. However, previous works focus exclusively on modeling edge variables, leaving open the problem of characterizing the joint distribution of variables across all levels of a simplicial complex through a principled, data-driven inference procedure.\\ 
\noindent\textbf{Contributions.} The key innovations of this paper are twofold. \emph{First}, differently from previous approaches, we propose SGM to jointly model random variables supported on all simplices of a simplicial complex through a \textit{single, parametrized} Gaussian distribution.  Building upon Hodge theory, SGM associates an independent random component with each simplex, thereby introducing uncertainty at every topological level. Consequently, the distribution for edge-related signals is derived as the marginal of the full SGM, thus inherently accounting for the latent contributions of node- and triangle-level variables. This differs from existing works, which instead focus on directly modeling edge-supported variables. 
\emph{Second}, assuming only the availability of edge-level signals, we then develop a statistically principled inference algorithm based on maximum-likelihood estimation to recover the full SGM parameters and, consequently, the conditional dependence structure of vertex, edge, and triangle variables. Unlike the simplicial lasso method in \cite{sardellitti2022probabilistic}, this approach performs a joint estimation of parameters over all levels of the complex. We assess the performance of our algorithm via synthetic simulations, demonstrating its effectiveness. 

%% file: Sections/background.tex
\vspace{-0.15cm}\section{BACKGROUND}\vspace{-0.2cm}
\label{sec:background}
\textbf{Simplicial Complexes.}
A \textit{simplicial complex} $\mathcal{X}$ is a pair $(\mathcal{V},\mathcal{S})$, where $\mathcal{V}$ is a set of vertices and $\mathcal{S}$ a family of subsets of $\mathcal{V}$ such that:  
(i) for every $v \in \mathcal{V}$, $\{v\} \in \mathcal{S}$;  
(ii) if $\tau \in \mathcal{S}$ and $\sigma \subset \tau$, then $\sigma \in \mathcal{S}$.  
An element of $\mathcal{S}$ with $k+1$ vertices is called a $k$-\textit{simplex}, or simplex of order $k$.
The largest order of simplices in a simplicial complex defines its \textit{dimension}. 
 By assigning one of the two possible orientations to each simplex of a simplicial complex, its entire structure can be described through a set of incidence matrices $\{\mathbf{B}_k\}$ (see \cite{barbarossa2020topological}).  
Given a $k$-simplex $\sigma_i^{k}$ and a $(k-1)$-simplex $\sigma_j^{k-1}$, the incidence matrix entry is
\[
\mathbf{B}_k(i,j)=
\begin{cases}
  +1\ , & \text{if $\sigma_j^{k-1} \subset \sigma_i^{k}$ with coherent orientation},\\[2pt]
  -1\ , & \text{if $\sigma_j^{k-1} \subset \sigma_i^{k}$ with opposite orientation},\\[2pt]
   0\ , & \text{otherwise}.
\end{cases}
\]
\textbf{Hodge Laplacians.} From the incidence matrices, we construct the \textit{combinatorial Laplacians}, which generalize the graph Laplacian to simplicial complexes.  
They consist of the \textit{lower Laplacian} $\mathbf{L}_{k,d} = \mathbf{B}_k^\top \mathbf{B}_k$ and the \textit{upper Laplacian} $\mathbf{L}_{k,u} = \mathbf{B}_{k+1}\mathbf{B}_{k+1}^\top$,  
where $k$ denotes the simplex order.  
These matrices encode adjacency relations among $k$-simplices:  
$\mathbf{L}_{k,d}$ captures \textit{lower adjacency}, i.e., two $k$-simplices share a common $(k-1)$-face,  
whereas $\mathbf{L}_{k,u}$ captures \textit{upper adjacency}, i.e., two $k$-simplices are cofaces of a common $(k+1)$-simplex.  
Their sum $\mathbf{L}_k = \mathbf{L}_{k,d} + \mathbf{L}_{k,u}$ is known as the \textit{Hodge $k$-Laplacian}. \\
\textbf{Topological Signals.} 
From now on, w.l.o.g., we focus on $2$-dimensional simplicial complexes $\mathcal{X} = (\mathcal{V}, \mathcal{E}, \mathcal{T})$, 
where $\mathcal{V}$, $\mathcal{E}$, and $\mathcal{T}$ denote the sets of vertices, edges, and triangles, respectively.  
A \textit{vertex signal} is a mapping $\mathbf{x}_V : \mathcal{V} \rightarrow \mathbb{R}$,
which, upon fixing an ordering of the vertices, can be represented as a vector 
$\mathbf{x}_V \in \mathbb{R}^{|\mathcal{V}|}$.  
Analogously, we define \textit{edge signals}  and \textit{triangle signals} $\mathbf{x}_E \in \mathbb{R}^{|\mathcal{E}|}, \
\mathbf{x}_T \in \mathbb{R}^{|\mathcal{T}|}$,
obtained by assigning an ordering to edges and triangles, respectively.  
When these signals are stochastic, we denote them by uppercase letters $X_V$, $X_E$, and $X_T$, and interpret them as random vectors in the corresponding Euclidean spaces.\\
\textbf{Hodge Decomposition.} 
The property $\mathbf{B}_k \mathbf{B}_{k+1} = \mathbf{0}$ of the incidence matrices induces an orthogonal \textit{Hodge decomposition} of the signal spaces.  
For a $2$-dimensional simplicial complex, this yields
\begin{align}
\mathbb{R}^{|\mathcal{V}|} &= 
\mathrm{ker}(\mathbf{L}_0)
\oplus
\mathrm{im}(\mathbf{B}_1)\ , \label{eq:Hodge_dec_1}\\[4pt]
\mathbb{R}^{|\mathcal{E}|} &= 
\mathrm{im}(\mathbf{B}_1^\top)
\oplus
\mathrm{ker}(\mathbf{L}_1)
\oplus
\mathrm{im}(\mathbf{B}_2) \ , \label{eq:Hodge_dec_2} \\[4pt]
\mathbb{R}^{|\mathcal{T}|} &= 
\mathrm{im}(\mathbf{B}_2^\top)
\oplus
\mathrm{ker}(\mathbf{L}_2) \ , \label{eq:Hodge_dec_3}
\end{align}
where $\mathrm{im}(\cdot)$ and $\mathrm{ker}(\cdot)$ denote the image and kernel of a linear operator. At edge level, this decomposition has a precise differential interpretation. Given an edge signal $\mathbf{x}_E \in \mathbb{R}^{|\mathcal{E}|}$, there exists three orthogonal vectors such that 
\begin{equation}
\mathbf{x}_E = \mathbf{B}_1^\top \mathbf{x}_V + \mathbf{B}_2 \mathbf{x}_T + \mathbf{h}_E.
\end{equation}
The first term is the \textit{irrotational} component, resulting from the \textit{discrete gradient} of a vertex signal $\mathbf{x}_V$. The second term is the \textit{solenoidal} component, and represents the divergence-free part of $\mathbf{x}_E$. Finally, the term $\mathbf{h}_E \in \ker(\mathbf{L}_1)$ is the \textit{harmonic} component, and reflects purely topological information.\\
\textbf{Gaussian Markov Random Fields.} 
An MRF is a set of random variables $\{X_i : i \in \mathcal{V}\}$ indexed by the nodes of an undirected graph 
$G = (\mathcal{V}, \mathcal{E})$ and satisfying the \textit{Markov property}:
\begin{equation}
X_i \perp\!\!\!\perp X_j \;|\; X_{\mathcal{V} \setminus \{i,j\}}
\quad \Longleftrightarrow \quad (i,j) \notin \mathcal{E},
\end{equation}
that is, two variables are conditionally independent given all the others if and only if 
no edge connects their nodes.  
An edge thus encodes a \textit{direct statistical interaction} between its endpoints. When the variables jointly follow a multivariate Gaussian distribution 
$X \sim \mathcal{N}(\mu,\Sigma)$ with positive definite covariance matrix $\Sigma$, 
the model is called a \textit{Gaussian Markov Random Field} (GMRF).  
Without loss of generality, we assume $\mu = 0$ and parameterize the distribution through 
its \textit{precision matrix} $\bm{\Omega} = \Sigma^{-1}$.
This is because in GMRF the precision matrix encodes the conditional independence structure of the model \cite{rue2005gaussian}:
\begin{equation}
\bm{\Omega}_{ij} = 0 \quad \Longleftrightarrow \quad
X_i \perp\!\!\!\perp X_j \;|\; X_{V \setminus \{i,j\}},
\qquad i \neq j. 
\end{equation}
Hence, the sparsity pattern of $\bm{\Omega}$ defines the edges of the MRF.\\
\textbf{Schur Complement.} Consider a zero-mean Gaussian random vector $X = [Y, W]
\sim \mathcal{N}\!\bigl(0,\bm{\Omega}^{-1}\bigr)$, where 
\begin{equation}\label{eq:partitioned_precision}
\bm{\Omega} =
\begin{bmatrix}
\bm{\Omega}_{Y,Y} & \bm{\Omega}_{Y,W} \\
\bm{\Omega}_{Y,W}^\top& \bm{\Omega}_{W,W}
\end{bmatrix},  
\end{equation}
is the precision matrix of $X$, block-partitioned according to two groups of variables $Y$ and $W$. A key property of Gaussian distributions is that the precision matrix of $Y$, denoted $\bm{\Omega}_Y$, is given by 
\begin{equation}\label{eq:Schur_complement}
\bm{\Omega}_Y=\bm{\Omega}_{Y,Y} - \bm{\Omega}_{Y,W} \bm{\Omega}_{W,W}^{-1} \bm{\Omega}_{Y,W}^\top,
\end{equation}
which is called its \textit{Schur complement}.  
Hence, $Y \sim \mathcal{N}\!\bigl(0,\;\bm{\Omega}_Y^{-1} \bigr).$ Further, the block partition \eqref{eq:partitioned_precision} allows to derive the linear relation 
\begin{equation}\label{eq:linear_regression}
Y = -\bm{\Omega}_{Y,Y}^{-1} \bm{\Omega}_{Y,W}W + Z,
\qquad
Z \sim \mathcal{N}(0, \bm{\Omega}_{Y,Y}^{-1}),
\end{equation}
where $Z$ is an independent Gaussian innovation term. 

%% file: Sections/simplicial_gaussian_model.tex
\vspace{-0.15cm}\section{Simplicial Gaussian model}
Throughout this work, we focus on $2$-dimensional simplicial complexes $\mathcal{X}$, although our model 
extends to arbitrary order. 
Let $\{X_\sigma\}$ be a family of random variables associated with the vertices, edges, and triangles of $\mathcal{X}$, and the corresponding 
collections at each level denoted by the random vectors $X_V$, $X_E$, and $X_T$, respectively. 
We assume that these variables jointly follow a non-singular zero-mean Gaussian distribution, $X \sim \mathcal{N}(0, \bm{\Omega}^{-1})$, parameterized by the precision matrix $\bm{\Omega}$ encoding conditional independencies.
Since we build upon Hodge theory, $\bm{\Omega}$ has a block structure highlighting the role of incidence relations in shaping \emph{direct interactions} across variables. 
Specifically, in the SGM we have
\begin{equation}\label{eq:precision-matrix}
\bm{\Omega} =
\begin{bmatrix}
\mathbf{D}_V^{-1} & -\mathbf{B}_1 & 0 \\
-\mathbf{B}_1^\top & \mathbf{D}_E^{-1} & -\mathbf{B}_2 \\
0 & -\mathbf{B}_2^\top & \mathbf{D}_T^{-1}
\end{bmatrix}\,,
\end{equation}
with $\mathbf{D}_V \in \mathbb{R}^{|\mathcal{V}|\times |\mathcal{V}|}$, $\mathbf{D}_E \in \mathbb{R}^{|\mathcal{E}|\times |\mathcal{E}|}$, 
$\mathbf{D}_T \in \mathbb{R}^{|\mathcal{T}|\times |\mathcal{T}|}$ positive diagonal matrices, and 
$\mathbf{B}_1$,   $\mathbf{B}_2$ the incidence matrices of $\mathcal{X}$.
By \Cref{eq:linear_regression}, we can express each random vector in terms of the others plus an independent Gaussian term:
\begin{equation}\label{eq:SGMdecomposition}
    \begin{aligned}
        X_V &= \mathbf{D}_V \mathbf{B}_1 X_E + Z_V, \quad Z_V \sim \mathcal{N}(0, \mathbf{D}_V)\,, \\
        X_E &= \mathbf{D}_E\mathbf{B}_1^\top X_V + \mathbf{D}_E\mathbf{B}_2 X_T + Z_E, \quad Z_E \sim \mathcal{N}(0, \mathbf{D}_E)\,, \\
        X_T &= \mathbf{D}_T \mathbf{B}_2^\top X_E + Z_T, \quad Z_T \sim \mathcal{N}(0, \mathbf{D}_T)\,. 
    \end{aligned}    
\end{equation}
\Cref{eq:SGMdecomposition} shows that the SGM induces an Hodge-like decomposition of the random signals at any level of the complex. 
The incidence matrices, rescaled by the diagonal weights $\mathbf{D}_V, \mathbf{D}_E, \mathbf{D}_T$, encode the deterministic dependence on adjacent levels, 
whereas the Gaussian terms $Z_V$, $Z_E$, and $Z_T$ provide independent stochastic fluctuations associated with each simplex of $\mathcal{X}$. Note that the harmonic components of the classical Hodge decomposition 
(see \Cref{eq:Hodge_dec_1,eq:Hodge_dec_2,eq:Hodge_dec_3}) are not present in the SGM, 
as they encode purely topological information and do not couple signals across levels, 
thus not affecting the modeled conditional dependencies.

\spara{Edge-level model.} 
Motivated by applications (cf. \Cref{sec:intro}), we analyze the marginal distribution induced by the SGM on the edge variables, i.e., the random vector $X_E$. 
Random variables on nodes and triangles are latent, and we can infer their contribution to the SGM exclusively through their impact on the edge distribution. Since we expect edge variability to arise primarily from vertex and triangle contributions, we assume spatially homogeneous edge noise, so that residual fluctuations are modeled as independent of edge location. 
\begin{assumption}\label{ass:homvar}
    The variance of the fluctuations is homogeneous across edges, that is, $\mathbf{D}^{-1}_E = k \mathbf{I}$ for some $k>0$.     
\end{assumption} 
\begin{proposition}
Under the model defined in~\Cref{eq:precision-matrix} and Assumption 1, the marginal precision matrix of the edge vector $X_E$ is
\begin{equation}\label{eq:precision_E}
 \bm{\Omega}_E \;=\; k\,\mathbf{I}_E - \mathbf{B}_1^\top \mathbf{D}_V \mathbf{B}_1 
- \mathbf{B}_2 \mathbf{D}_T \mathbf{B}_2^\top.
\end{equation}

\end{proposition}
\begin{proof}
Up to a permutation of variables, consider the partition of the random vector 
$\mathbf{X} = [Y, W]$, with $Y := X_E$ and 
$W := [X_V, X_T]$.  
Given this reordering, the precision matrix $\bm{\Omega}$ in \Cref{eq:precision-matrix} can be written in block form as in \Cref{eq:partitioned_precision}, with 
$\bm{\Omega}_{Y,Y} = \mathbf{D}_E^{-1}$, 
$\bm{\Omega}_{Y,W} = [-\mathbf{B}_1^\top, -\mathbf{B}_2 ]$, and
$\bm{\Omega}_{W,W} = \mathrm{diag}(\mathbf{D}_V^{-1}, \mathbf{D}_T^{-1})$.\\
Applying the Schur complement formula \Cref{eq:Schur_complement}, we obtain
\begin{equation}
\bm{\Omega}_E 
= \bm{\Omega}_{Y,Y} - \bm{\Omega}_{Y,W}\,\bm{\Omega}_{W,W}^{-1}\,\bm{\Omega}_{Y,W}^\top.
\end{equation}
Thus, by simple substitutions and block-matrix multiplication we get
\begin{equation}\label{eq:OmegaE_schur}
    \bm{\Omega}_E = \mathbf{D}_E^{-1} - \mathbf{B}_1^\top \mathbf{D}_V \mathbf{B}_1 - \mathbf{B}_2 \mathbf{D}_T \mathbf{B}_2^\top \,. 
\end{equation}
Finally, using Assumption~1 gives 
\Cref{eq:precision_E}.
\end{proof}
According to \Cref{eq:precision_E}, the non-zero entries in $\bm{\Omega}_E$ are constrained by the simplicial complex topology. 
The second term in \Cref{eq:OmegaE_schur} captures dependencies between edges that share a common vertex. 
These dependencies arise since the vertex carries a latent stochastic component that influences all incident edges. 
Similarly, the third term captures dependencies between edges that are co-faces of the same triangle, since on the triangle resides a latent stochastic component that simultaneously affects all its edges. 
\begin{algorithm}
\caption{Edge-level Inference}
\begin{algorithmic}
\State \textbf{Input:} Edge signals $\mathbf{x}_E[1],\dots,\mathbf{x}_E[M]$; compute 
$\mathbf{C} = \tfrac{1}{M}\sum_{i=1}^M \mathbf{x}_E[i]\mathbf{x}_E[i]^\top$.
\State \textbf{Initialize:} $k^{(0)}$, $\tilde{\mathbf{d}}_V^{(0)}$, 
$\tilde{\mathbf{d}}_T^{(0)}$ satisfying (a) and (b).
\For{$n=0,1,2,\dots$ until convergence}
    \State Update $k^{(n+1)}$ by solving
    \[\max_{k>0} 
\; N_E \log k - k\, \operatorname{tr}(\mathbf{C}) + f_V(\tilde{\mathbf{d}}_V^{(n)},k) + f_T(\tilde{\mathbf{d}}_T^{(n)},k);\]
    \State Update $\tilde{\mathbf{d}}_V^{(n+1)}$ by solving
    \[\max_{\tilde{\mathbf{d}}_V > 0,\, \text{(a)}} f_V(\tilde{\mathbf{d}}_V, k^{(n+1)});\]
    \State Update $\tilde{\mathbf{d}}_T^{(n+1)}$ by solving
    \[\max_{\tilde{\mathbf{d}}_T \ge 0,\, \text{(b)}} f_T(\tilde{\mathbf{d}}_T, k^{(n+1)});\]
\EndFor
\State \textbf{Output:} Estimates $\hat{k}$, $\hat{\mathbf{d}}_V$, $\hat{\mathbf{d}}_T$.
\end{algorithmic}\label{alg:inference}
\end{algorithm}

\spara{Inference algorithm for $\bm{\Omega}_E$.} 
Let $\mathbf{x}_E[i]$, $i=1,\dots,M$, denote realizations of $X_E \sim \mathcal{N}(0,\bm{\Omega}_E^{-1})$. 
Since edge signals are observed, we assume w.l.o.g. the underlying graph is known.
Conversely, the set of triangles must be inferred from data together with $\bm{\Omega}_E$. 
To this end, we start by assuming all $3$-cliques in the graph filled with triangles, and we build the corresponding incidence matrix $\mathbf{B}_2$.
We then consider the maximum likelihood criterion for estimating $\bm{\Omega}_E$:
\begin{equation}\label{prob:P1}
\hat{\bm{\Omega}}_E = \arg\max_{\bm{\Theta}_E \succ 0} 
\Big( \log\det(\bm{\Theta}_E) - \operatorname{tr}(\mathbf{C} \, \bm{\Theta}_E) \Big)\,,
\tag{P1}
\end{equation}
where $\mathbf{C} = \frac{1}{M} \sum_{i=1}^{M} \mathbf{x}_E[i]\mathbf{x}_E[i]^\top$ is the sample covariance matrix. 
By \Cref{eq:precision_E}, inference reduces to estimating the diagonal entries of $\mathbf{D}_V$ and $\mathbf{D}_T$, viz. $\mathbf{d}_V$ and $\mathbf{d}_T$, 
and a parameter $k$.  
Importantly, null entries---or below a given threshold---in the estimated $\mathbf{d}_T$ indicate the absence of triangles, thus enabling joint learning of parameters and topology.
Hence, \eqref{prob:P1} becomes:
\begin{equation}\label{opt_problem_compact}
    \begin{aligned}
        &\max_{\mathbf{d}_T\geq0;\mathbf{d}_V,k >0}
        \; \log\det\!\big(\bm{\Omega}_E(\mathbf{d}_V,\mathbf{d}_T,k)\big)
        - \operatorname{tr}\!\big(\mathbf{C}\,\bm{\Omega}_E(\mathbf{d}_V,\mathbf{d}_T,k)\big) \; \notag  \\[4pt]
        &\quad\text{s. t.:} 
       \qquad \bm{\Omega}_E(\mathbf{d}_V,\mathbf{d}_T,k) \succ 0\,. 
    \end{aligned}
    \tag{P2}
\end{equation}
Note that the objective function is strictly concave in $\bm{\Omega}_E$---an affine function of the parameters $(\mathbf{d}_V,\mathbf{d}_T,k)$---and the feasible set is convex. 
Hence, the  optimization problem is convex and can be solved using classical numerical methods \cite{CVX}. 

However, the intricate structure of the objective function makes the optimization problem challenging in practice. 
To address this issue, we propose a numerically efficient block–coordinate algorithm based on a reparameterization of the variables. 
Specifically, we introduce $\tilde{\mathbf{d}}_V = \tfrac{1}{k}\,\mathbf{d}_V$ and $\tilde{\mathbf{d}}_T = \tfrac{1}{k}\,\mathbf{d}_T$. 
Using the composition property $\mathbf{B}_2 \mathbf{B}_1 = 0$, the edge-level precision matrix factorizes as
\begin{equation}\label{eq:dec_precision_E}
\bm{\Omega}_E 
= k\big(\mathbf{I}_E - \mathbf{B}_1^\top \operatorname{diag}(\tilde{\mathbf{d}}_V)\mathbf{B}_1\big)
   \big(\mathbf{I}_E - \mathbf{B}_2 \operatorname{diag}(\tilde{\mathbf{d}}_T)\mathbf{B}_2^\top\big)\,,
\end{equation}
which allows us to decompose the log-determinant term.  
Accordingly, \eqref{opt_problem_compact} can be reformulated as
\begin{equation}
    \begin{aligned}\label{opt_problem_ext}
    &\max_{\mathbf{d}_T\geq0;\mathbf{d}_V,k >0} 
    \; N_E \log k - k\, \operatorname{tr}(\mathbf{C}) + f_V(\tilde{\mathbf{d}}_V,k) + f_T(\tilde{\mathbf{d}}_T,k) \\
    &\quad\text{s.t.:} 
    \quad\text{(a)}\;\; \mathbf{I}_E - \mathbf{B}_1^\top \operatorname{diag}(\tilde{\mathbf{d}}_V)\mathbf{B}_1 \succ 0\,, \notag\\
    &\quad \ \  \qquad\text{(b)}\;\; \mathbf{I}_E - \mathbf{B}_2 \operatorname{diag}(\tilde{\mathbf{d}}_T)\mathbf{B}_2^\top \succ 0\,, \notag
    \end{aligned}
    \tag{P3}
\end{equation}
with
\begin{align}
f_V(\tilde{\mathbf{d}}_V,k) 
&= \log\det\!\big(\mathbf{I}_E - \mathbf{B}_1^\top \operatorname{diag}(\tilde{\mathbf{d}}_V)\mathbf{B}_1\big) + \notag\\
&\qquad\qquad - k\, \operatorname{tr}\!\Big(\mathbf{C}\, \mathbf{B}_1^\top \operatorname{diag}(\tilde{\mathbf{d}}_V)\mathbf{B}_1\Big)\,, \notag\\[2mm]
f_T(\tilde{\mathbf{d}}_T,k) 
&= \log\det\!\big(\mathbf{I}_E - \mathbf{B}_2 \operatorname{diag}(\tilde{\mathbf{d}}_T)\mathbf{B}_2^\top\big) +\notag\\
&\qquad\qquad - k\, \operatorname{tr}\!\Big(\mathbf{C}\, \mathbf{B}_2 \operatorname{diag}(\tilde{\mathbf{d}}_T)\mathbf{B}_2^\top\Big)\,. \notag
\end{align}
Note that by \Cref{eq:dec_precision_E} the constraint in \eqref{opt_problem_compact} is equivalent to the pair of conditions (a) and (b) in \eqref{opt_problem_ext}.
The reformulation in \eqref{opt_problem_ext} reveals a natural scale separation, which directly motivates the use of a block--coordinate method. 
The main steps of the optimization method are outlined in \Cref{alg:inference}. The optimization alternates between maximizing the objective with respect to variables $k$, $\mathbf{d}_V\geq0$, and $\mathbf{d}_T$. In particular, step 4 admits a closed-form solution (details omitted for brevity), while the remaining sub-problems can be solved using standard convex optimization algorithms. 

%% file: Sections/numerical_results.tex
\begin{figure*}[!t]
    \centering
    \includegraphics[width=\textwidth]{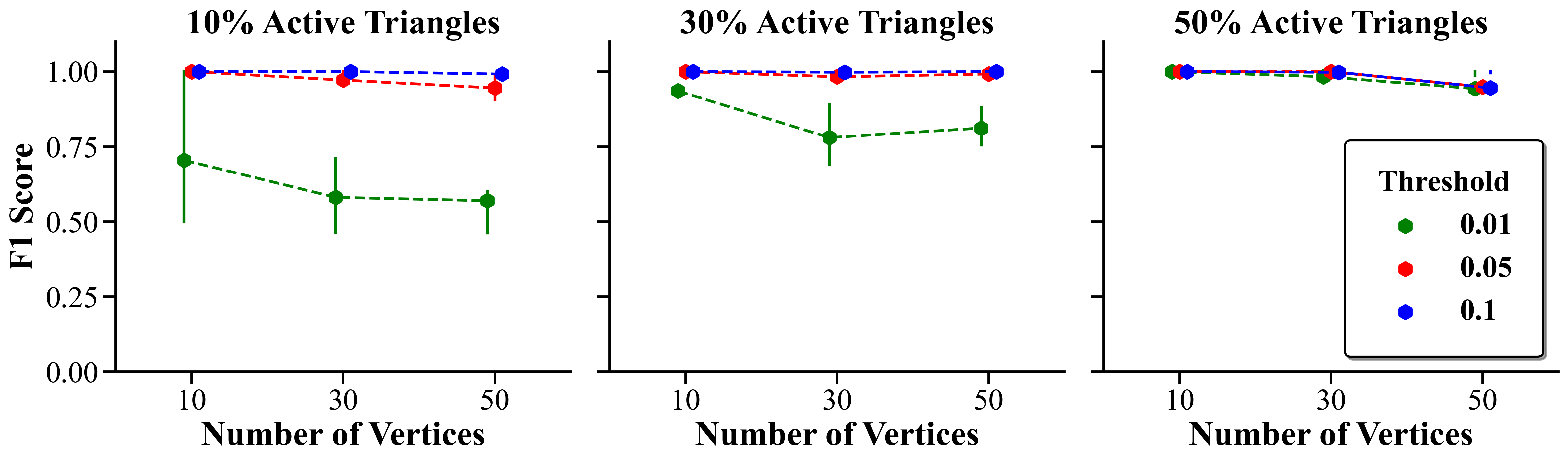}
    \caption{F1-score for triangle detection after thresholding the estimated 
    parameters $\hat{\mathbf{d}}_T$ at three different levels ($0.01$, $0.05$, $0.1$).}
    \label{fig:f1score}
\end{figure*}
\begin{figure}[!t]
    \centering
    \includegraphics[width = 0.95\linewidth]{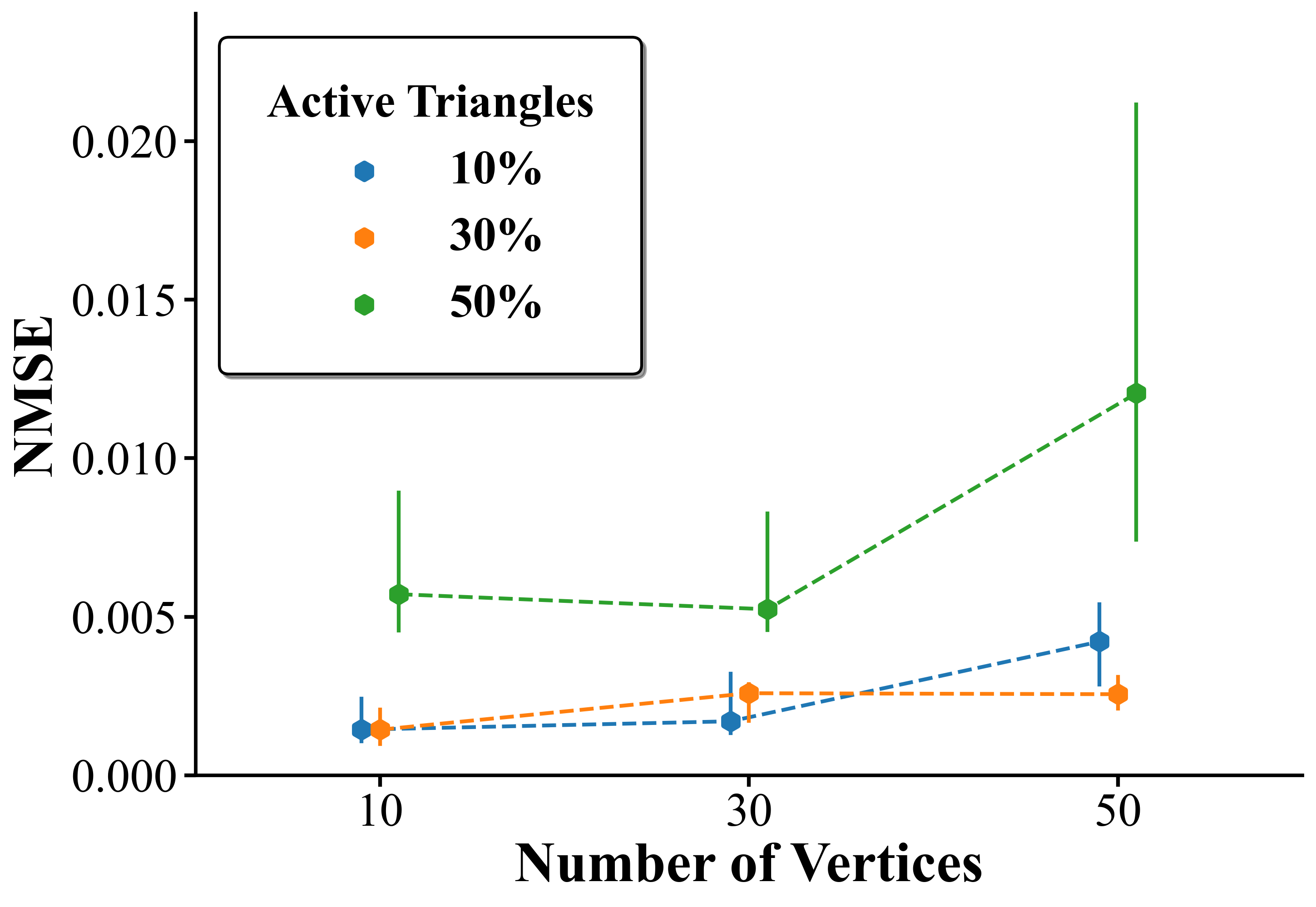}
\caption{NMSE for $(k,\mathbf{d}_V,\mathbf{d}_T)$ vs number of vertices.}
    \label{fig:synth}
\end{figure}
\vspace{-0.2cm}\section{NUMERICAL RESULTS}
\label{sec:num_results}
We test \Cref{alg:inference} on random $2$-dimensional simplicial complexes, generated as follows.  
We build a random graph with $|V|$ nodes and edge probability $q=0.3$, and fill a percentage $p$ of its $3$-cliques with triangles.
To study the sensitivity w.r.t. size and sparsity of the complex, we repeat this procedure for $(|V|, p) \in \{10, 30, 50\} \times \{10\%,30\%,50\%\}$, simulating for each setting $20$ different complexes.
For each complex, we then build a block-structured precision matrix $\bm{\Omega}$ as in \Cref{eq:precision-matrix} by specifying a positive parameter for each vertex and triangle, collected in the vectors $\mathbf{d}_V$ and $\mathbf{d}_T$, respectively, and a common positive parameter $k$ for all edges.  
The entries of $\mathbf{d}_V$ and $\mathbf{d}_T$ are sampled uniformly at random in $[0.2,1]$, while $k$ is chosen sufficiently large to ensure that $\bm{\Omega} \succ 0$.
We then draw $50000$ i.i.d. samples from $\mathcal{N}(0,\bm{\Omega}^{-1})$, thus generating observations on vertices, edges, and triangles.
Our goal is to recover $k$, $\mathbf{d}_V$, and $\mathbf{d}_T$ from the edge observations using \Cref{alg:inference}.
The monitored metrics are: \emph{(i)} the F1 score to evaluate the performance in the retrieval of active triangles; and \emph{(ii)} the normalized mean squared error (NMSE) 
\begin{equation}\label{eq:nmse}
    \mathrm{NMSE} = \frac{\|\hat{\mathbf{d}}_V - \mathbf{d}_V\|_2^2
      + \|\hat{\mathbf{d}}_T - \mathbf{d}_T\|_2^2
      + |\hat{k} - k|^2}{\|\mathbf{d}_V\|_2^2 + \|\mathbf{d}_T\|_2^2 + k^2}\,,
\end{equation}
where $\hat{\mathbf{d}}_V$, $\hat{\mathbf{d}}_T$, and $\hat{k}$ are the estimates.
To account for small fluctuations arising from finite samples, when computing the F1 we perform a final pruning step by hard-thresholding $\hat{\mathbf{d}}_T$ at $0.01$, $0.05$, and $0.1$.

\Cref{fig:f1score} shows the F1 score, where panels refer to the values of $p$ and the x-axis to the size of the complex. 
Further, vertical bars are the interquartile ranges computed across the $20$ simulations.
Overall, \Cref{alg:inference} correctly identifies active triangles among all the $3$-cliques of the graph, i.e., the simplices where the entries of $\hat{\mathbf{d}}_T$ are non null.
Specifically, the F1 increases with the threshold, and \Cref{alg:inference} reaches nearly perfect detection at $0.05$.
Hence, the estimates for absent triangles are mainly numerical fluctuations.
Additionally, in case of sparser complexes, the benefits of hard-thresholding are higher. Finally, \Cref{fig:synth} reports the NMSE for estimating $(k, \mathbf{d}_V, \mathbf{d}_T)$, showing that \Cref{alg:inference} achieves low error across all settings.

%% file: Sections/conclusions.tex
\vspace{-0.2cm}\section{Conclusions}
\label{sec:conclusions}
This work introduced SGMs for modeling random variables supported on simplicial complex within a single, parametrized Gaussian model.  
By adding independent random components at each topological level, SGM extends Hodge-theoretic decompositions to account for uncertainty.  
Motivated by practical applications, we derive the edge-level marginal distribution from the SGM, which captures latent vertex and triangle effects. Building on this, we formulate a convex optimization problem to estimate the parameters of the associated precision matrix.
Our algorithmic solution hinges on block-coordinate optimization, and it accurately recovers both the latent triangles and the statistical dependencies from edge observations. As future works, we plan to relax Assumption 1 and to move beyond synthetic data to fully assess the potential of SGMs in real-world scenarios. We will also compare the performance of \Cref{alg:inference} with the inference algorithm proposed in \cite{sardellitti2022probabilistic} in estimating higher-order topology and the overall data distribution.
Additionally, SGM naturally opens to a factorization of the joint distribution shaped by the incidence relations over the complex, which we regard as an important future research direction.